\documentclass[12pt]{article}
\usepackage[centertags]{amsmath}
\usepackage{amsfonts}
\usepackage{amssymb}
\usepackage{latexsym}
\usepackage{amsthm}
\usepackage{newlfont}
\usepackage{graphicx}
\usepackage{listings}
\usepackage{booktabs}
\usepackage{abstract}
\newlength{\defbaselineskip}
\newcommand{\setlinespacing}[1]%
           {\setlength{\baselineskip}{#1 \defbaselineskip}}

\newcommand{\actaqed}{\hfill $\actabox$}
{\medskip\noindent \textit{Proof of #1. }}%
{\actaqed \medskip}

\def \bx{{\mathbf x}}

\def\mb{{\bar m}}
\def\Mb{{\bar M}}
\def\Nb{{\bar N}}
\def\D{{\mathcal D}}

\def\U{{\mathcal U}}
\def\V{{\mathcal V}}

\def\Th{{\Theta}}
\def \<{\langle}
\def\>{\rangle}

\def \ep{\epsilon}
\def \de{\delta}

\def \ff{\varphi}

\def\bN{{\mathbf N}}

\def \sp{\operatorname{span}}

\def\bt{\beta}
\def\la{\lambda}

\newtheorem{Theorem}{Theorem}[section]
\newtheorem{Lemma}{Lemma}[section]

\newtheorem{Proposition}{Proposition}[section]
\newtheorem{Remark}{Remark}[section]

\newtheorem{Corollary}{Corollary}[section]
\numberwithin{equation}{section}

\begin{document}
\title{{Nonlinear tensor product approximation of functions}\thanks{\it Math Subject Classifications.
primary:  41A65; secondary: 41A25, 41A46, 46B20.}}
\author{D. Bazarkhanov \thanks{Institute of Mathematics and Mathematical Modeling, Kazakhstan} and V. Temlyakov \thanks{ University of South Carolina, USA, and Steklov Institute of Mathematics, Russia. Research was supported by NSF grant DMS-1160841 }} \maketitle
\begin{abstract}
{We are interested in approximation of a multivariate function $f(x_1,\dots,x_d)$ by linear combinations of products $u^1(x_1)\cdots u^d(x_d)$ of univariate functions $u^i(x_i)$, $i=1,\dots,d$. In the case $d=2$ it is a classical problem of bilinear approximation. In the case of approximation in the $L_2$ space the bilinear approximation problem is closely related to the problem of singular value decomposition (also called Schmidt expansion) of the corresponding integral operator with the kernel $f(x_1,x_2)$. There are known results on the rate of decay of errors of best bilinear approximation in $L_p$ under different smoothness assumptions on $f$. The problem of multilinear approximation (nonlinear tensor product approximation) in the case $d\ge 3$ is more difficult and much less studied than the bilinear approximation problem. We will present   results on best multilinear approximation in $L_p$ under mixed smoothness assumption on $f$. }
\end{abstract}

\section{Introduction}

In this paper we study multilinear approximation (nonlinear tensor product approximation) of functions.
 For a function $f(x_1,\dots,x_d)$ denote
$$
\Th_M(f)_X:=\inf_{\{u^i_j\},  j=1,\dots,M, i=1,\dots,d}\|f(x_1,\dots,x_d) - \sum_{j=1}^M\prod_{i=1}^d u^i_j(x_i)\|_X
$$
and for a function class $F$ define
$$
\Th_M(F)_X:=\sup_{f\in F} \Th_M(f)_X.
$$
In the case $X=L_p$ we write $p$ instead of $L_p$ in the notation.
In other words we are interested in studying $M$-term approximations of functions with respect to the dictionary
$$
\Pi^d := \{g(x_1,\dots,x_d): g(x_1,\dots,x_d)=\prod_{i=1}^d u^i(x_i)\}
$$
where $u^i(x_i)$ are arbitrary univariate functions. 
We discuss the case of $2\pi$-periodic functions of $d$ variables and approximate them in the $L_p$ spaces. Denote by $\Pi^d_p$ the normalized in $L_p$ dictionary $\Pi^d$ of $2\pi$-periodic functions. We say that a dictionary $\D$ has a tensor product structure if all its elements have a form of products $u^1(x_1)\cdots u^d(x_d)$ of univariate functions $u^i(x_i)$, $i=1,\dots,d$. Then any dictionary with tensor product structure is a subset of $\Pi^d$. 
The classical example of a dictionary with tensor product structure is the $d$-variate trigonometric system $\{e^{i(k,x)}\}$. Other examples include the hyperbolic wavelets and the hyperbolic wavelet type system $\U^d$ defined in Section 3. 

The nonlinear tensor product approximation is very important in numerical applications. We refer the reader to the monograph \cite{H} which presents the state of the art on the topic. Also, the reader can find a very recent discussion of related results in \cite{SU}. 

In the case $d=2$ the multilinear approximation problem is a classical problem of bilinear approximation. In the case of approximation in the $L_2$ space the bilinear approximation problem is closely related to the problem of singular value decomposition (also called Schmidt expansion) of the corresponding integral operator with the kernel $f(x_1,x_2)$. There are known results on the rate of decay of errors of best bilinear approximation in $L_p$ under different smoothness assumptions on $f$. We only mention some known results for classes of functions which are studied in this paper. We study the classes $W^r_q$ of functions with bounded mixed derivative which we define for positive $r$
(not necessarily an integer). Let
$$
F_r(t) := 1+2\sum_{k=1}^\infty k^{-r}\cos(kt -\pi r/2)
$$
be the univariate Bernoulli kernel and let
$$
F_r(x) := F_r(x_1,\dots,x_d) := \prod_{i=1}^d F_r(x_i)
$$
be its multivariate analog. We define
$$
W^r_q := \{f: f= F_r \ast \ff, \quad \|\ff\|_q \le 1\},
$$
where $\ast$ denotes the convolution.

The problem of estimating $\Th_M(f)_2$ in case $d=2$ (best $M$-term bilinear approximation in $L_2$) is a classical one and was considered for the first time by E. Schmidt \cite{S} in 1907. For many function classes $F$ an asymptotic behavior of
$\Th_M(F)_p$ is known. For instance, the relation
\begin{equation}\label{1.1}
\Th_M(W^r_q)_p \asymp M^{-2r + (1/q-\max(1/2,1/p))_+}
\end{equation}
for $r>1$ and $1\le q\le p \le \infty$ follows from more general results in \cite{T32}.
In the case $d>2$ almost nothing is known. There is (see \cite{T35}) an upper estimate in the case $q=p=2$
\begin{equation}\label{1.2}
\Th_M(W^r_2)_2 \ll M^{-rd/(d-1)} .  
\end{equation}

Results of this paper are around the bound (\ref{1.2}). First of all we discuss the lower bound 
matching the upper bound (\ref{1.2}). In the case $d=2$ the lower bound 
\begin{equation}\label{1.3}
\Th_M(W^r_p)_p \gg M^{-2r},\qquad 1\le p\le \infty ,  
\end{equation}
follows from more general results in \cite{T32} (see (\ref{1.1}) above). A stronger result 
\begin{equation}\label{1.4}
\Th_M(W^r_\infty)_1 \gg M^{-2r}  
\end{equation}
follows from Theorem 1.1 in \cite{T46}. 

We could not prove the lower bound matching the upper bound (\ref{1.2}) for $d>2$. Instead, we prove a weaker lower bound. For a function $f(x_1,\dots,x_d)$ denote
$$
\Th^b_M(f)_X:=\inf_{\{u^i_j\}, \|u^i_j\|_X\le b\|f\|_X^{1/d}}\|f(x_1,\dots,x_d) - \sum_{j=1}^M\prod_{i=1}^d u^i_j(x_i)\|_X
$$
and for a function class $F$ define
$$
\Th^b_M(F)_X:=\sup_{f\in F} \Th^b_M(f)_X.
$$
In Section 2 we prove the following lower bound (see Corollary \ref{C2.2})
$$
\Th^b_M(W^r_\infty)_{ 1} \gg (M\ln M)^{-\frac{rd}{d-1}}.
$$
This lower bound indicates that probably the exponent $\frac{rd}{d-1}$ is the right one in the power decay of the $\Th_M(W^r_p)_p$. 

Secondly, we discuss some upper bounds which extend the bound (\ref{1.2}). The relation (\ref{1.1}) shows that for $2\le p\le \infty$ in the case $d=2$ one has
\begin{equation}\label{1.5}
\Th_M(W^r_2)_p \ll M^{-2r} .  
\end{equation}

In Section 3 we extend (\ref{1.5}) for $d>2$.
\begin{Theorem}\label{T1.1} Let $2\le p<\infty$ and $r> (d-1)/d$. Then
$$
\Th_M(W^r_2)_p \ll  \left(\frac{M}{(\log M)^{d-1}}\right)^{-\frac{rd}{d-1}}.
$$
\end{Theorem}

The proof of Theorem \ref{T1.1} in Section 3 is not constructive. It goes by induction and uses 
a nonconstructive bound in the case $d=2$. In Section 4 we discuss constructive ways of building good multilinear approximations. The simplest way would be to use known results about $M$-term approximation with respect to special systems with tensor product structure. We illustrate this idea on the example of the system $\U^d$ defined and discussed in Section 3. 
We define a well-known Thresholding Greedy Algorithm with respect to a basis. It is convenient for us to enumerate the basis functions by dyadic  intervals. Assume a given system $\Psi$ of functions $\psi_I$ indexed by dyadic intervals
can be enumerated in such a way that $\{\psi_{I^j}\}_{j=1}^\infty$ is a basis for $L_p$. Then we define the greedy algorithm
 $G^p(\cdot,\Psi)$ as follows. Let
$$
f = \sum_{j=1}^\infty c_{I^j}(f,\Psi)\psi_{I^j} 
$$
and
$$
c_I(f,p,\Psi) := \|c_I(f,\Psi)\psi_I\|_p .
$$
Then $c_I(f,p,\Psi) \to 0$ as $|I| \to 0$. Denote $\Lambda_m$ a set of $m$ dyadic intervals
 $I$ such that
$$
\min_{I\in \Lambda_m} c_I(f,p,\Psi) \ge \max_{J \notin \Lambda_m}c_J(f,p,\Psi) .  
$$
We define $G^p(\cdot,\Psi)$ by formula
$$
G^p_m(f,\Psi) := \sum_{I\in \Lambda_m} c_I(f,\Psi) \psi_I .
$$

For a system (dictionary) of elements $\D$ define the best $M$-term approximation in $X$ as follows
$$
\sigma_M(f,\D)_X := \inf_{g_j\in \D, c_j, j=1,\dots,M} \|f-\sum_{j=1}^M c_jg_j\|_X.
$$
With this standard notation we have
$$
\Th_M(f)_p = \sigma_M(f,\Pi^d)_{L_p}.
$$

It is proved in \cite{T69} that for $1<q,p <\infty$ and big enough $r$
\begin{equation}\label{1.6}
\sup_{f\in W^r_q}\|f-G_M^p(f,\U^d)\|_p\asymp \sigma_M(W^r_q,\U^d)_p \asymp M^{-r}(\log M)^{(d-1)r}. 
\end{equation}
The above relation (\ref{1.6}) illustrates two phenomena: (I) for the class $W^r_q$ the simple Thresholding Greedy Algorithm provides near best $M$-term approximation; (II) the rate $M^{-r}(\log M)^{(d-1)r}$ of best $M$-term approximation with respect to the basis $\U^d$, which has a tensor product structure, is not as good as best $M$-term approximation with respect to $\Pi^d$ (we have exponent $r$ for $\U^d$ instead of $\frac{rd}{d-1}$ for $\Pi^d$). 

In Section 4 we use two very different greedy-type algorithms to provide a constructive multilinear approximant. Surprisingly, these two algorithms give the same error bound. For instance, Theorems \ref{T4.3} and \ref{T4.4} give for big enough $r$ the following constructive upper bound for $2\le p<\infty$
$$
\Th_M(W^r_2)_p \ll \left(\frac{M}{(\ln M)^{d-1}}\right)^{-\frac{rd}{d-1} +\frac{\bt}{d-1}},\quad \bt:=\frac{1}{2}-\frac{1}{p} .
$$
This constructive upper bound has an extra term $\frac{\bt}{d-1}$ in the exponent compared to the best $M$-term approximation. It would be interesting to find a constructive way to obtain the near best approximation in this case. 

\section{The lower bound}

Let $X$ be a Banach space and let $B_X$ denote the unit ball of $X$ with the center at $0$. Denote by $B_X(y,r)$ a ball with center $y$ and radius $r$: $\{x\in X:\|x-y\|\le r\}$. For a compact set $A$ and a positive number $\ep$ we define the covering number $N_\ep(A)$
 as follows
$$
N_\ep(A) := N_\ep(A,X)  
$$
$$
:=\min \{n : \exists y^1,\dots,y^n :A\subseteq \cup_{j=1}^n B_X(y^j,\ep)\}.
$$

The following bound is well known (see, for instance, \cite{Tbook}, Ch. 3). 
\begin{Lemma}\label{L2.1} For any $n$-dimensional Banach space $X$ we have
$$
\ep^{-n} \le N_\ep(B_X,X) \le (1+2/\ep)^n.
$$
\end{Lemma}

For $\bN =(N_1,\dots, N_d)$ let $T(\bN)$ be the set of trigonometric polynomials of order $N_j$ in the $j$th variable. Denote
$$
T(\bN)_p:=\{t\in T(\bN): \|t\|_p\le 1\}
$$
and
$$
\Pi^d(\bN,n,b) :=\{f\in T(\bN)_2, \, f(\bx)=\sum_{j=1}^n u^1_j(x_1)\cdots u^d_j(x_d), \, \|u^i_j\|_2\le b\}.
$$
\begin{Lemma}\label{L2.2} We have
$$
N_\ep(\Pi^d(\bN,n,b),L_2) \le (C(b,d)/\ep)^{Cn\ln (n+1) \sum_{i=1}^d(2N_i+1)}, \quad 0<\ep \le 1.
$$
\end{Lemma}
\begin{proof} First of all it is clear that we can assume that $u^i_j \in T(N_i)$, $j=1,\dots,n$, $i=1,\dots, d$. Second, in the $b$-ball of the $T(N_i)_2$ we build a $\de$-net. It is known (see Lemma \ref{L2.1}) that we can build a net with cardinality $S_i$ satisfying
$$
S_i \le (Cb/\de)^{2N_i+1}.
$$
Third, for each $u^i_j(x_i)$ choose an $v_{s(i,j)}(x_i)$, $s(i,j)\in [1,S_i]$ from the corresponding $\de$-net such that 
$$
\|u^i_j(x_i)-v_{s(i,j)}(x_i)\|_2 \le \de.
$$
Then
$$
\|\prod_{i=1}^d u^i_j(x_i)-\prod_{i=1}^d v_{s(i,j)}(x_i)\|_2 \le db^{d-1} \de
$$
and
$$
\|\sum_{j=1}^n\prod_{i=1}^d u^i_j(x_i)-\sum_{j=1}^n\prod_{i=1}^d v_{s(i,j)}(x_i)\|_2 \le ndb^{d-1} \de.
$$
The total number of functions $\sum_{j=1}^n\prod_{i=1}^d v_{s(i,j)}(x_i)$ when $v_{s(i,j)}(x_i)$ are taken from sets of cardinalities $S_i$, $i=1,\dots, d$, does not exceed
$$
\left(\prod_{i=1}^d S_i\right)^n \le (Cb/\de)^{n\sum_{i=1}^d(2N_i+1)}.
$$
Specifying $\de = \frac{\ep}{ndb^{d-1}}$ we obtain
$$
 (Cb/\de)^{n\sum_{i=1}^d(2N_i+1)} \le n^{n\sum_{i=1}^d(2N_i+1)} \left(\frac{C(b,d)}{\ep}\right)^{n\sum_{i=1}^d(2N_i+1)}
$$
which completes the proof.

\end{proof}

We are interested in lower bounds for the following quantities. For a function $f(x_1,\dots,x_d)$ denote
$$
\Th^b_M(f)_X:=\inf_{\{u^i_j\}, \|u^i_j\|_X\le b\|f\|_X^{1/d}}\|f(x_1,\dots,x_d) - \sum_{j=1}^M\prod_{i=1}^d u^i_j(x_i)\|_X
$$
and for a function class $F$ define
$$
\Th^b_M(F)_X:=\sup_{f\in F} \Th^b_M(f)_X.
$$

\begin{Theorem}\label{T2.1} Let $N_1=\cdots =N_d = N$. There is $c(b,d)>0$ such that for any $M$ satisfying $M\ln M \le c(b,d)N^{d-1}$ there exists an $f\in T(\bN)_\infty$ with the property: for any $u^i_j(x_i)$, $\|u^i_j\|_1\le b$, we have
$$
\|f(\bx)-\sum_{j=1}^M\prod_{i=1}^d u^i_j(x_i)\|_1\ge C(b,d)>0.
$$
\end{Theorem}
\begin{proof} The proof repeats the proof of Theorem 1.1 from \cite{T46}. We use notations from \cite{T46}. Denoting 
$$
\ep:=\Th^b_M(T(\bN)_\infty)_1
$$
we prove, using Lemma \ref{L2.2}, in the same way as in \cite{T46} the following bound
\begin{equation}\label{2.1}
N_{2\ep}(K_\bN(T(\bN)_\infty)_2 \le C_1(b,d)^{N^d}(C_2(b,d)/\ep)^{C_3(d)NM\ln M},\quad N>0.
\end{equation}
Lemma 1.2 from \cite{T46} gives the lower bound
\begin{equation}\label{2.2}
N_{2\ep}(K_\bN(T(\bN)_\infty)_1 \ge (C(d)/\ep)^{N^d},\quad N>0.
\end{equation}
Comparing (\ref{2.1}) and (\ref{2.2}) we complete the proof of Theorem \ref{T2.1}.
\end{proof}
\begin{Corollary}\label{C2.1} Let $N_1=\cdots =N_d = N$. There is $c(b,d)>0$ such that for any $M$ satisfying $M\ln M \le c(b,d)N^{d-1}$ we have
$$
\Theta_M^b( T(\bN)_\infty)_1\ge C(b,d)>0.
$$
\end{Corollary}
\begin{proof} By the definition of $\Th^b_M(f)_1$ for all $f\in  T(\bN)_\infty$ we can only use $u^i_j$ satisfying the condition
$$
\|u^i_j\|_1 \le b\|f\|_1^{1/d} \le b\|f\|_\infty^{1/d} \le b.
$$
Therefore, Theorem \ref{T2.1} implies Corollary \ref{C2.1}.
\end{proof}

\begin{Corollary}\label{C2.2} One has
$$
\Th^b_M(W^r_\infty)_{ L_1} \gg (M\ln M)^{-\frac{rd}{d-1}}.
$$
\end{Corollary}
\begin{proof} By the Bernstein inequality
$$
CN^{-rd}T(\bN)_\infty \subset W^r_\infty.
$$
By Theorem \ref{T2.1} with $N\asymp (M\ln M)^{\frac{1}{d-1}}$ we obtain the required bound.
\end{proof} 

\section{Upper bounds. Proof of Theorem \ref{T1.1}}

We define the system $\U:=\{U_I\}$ in the univariate case. Denote
$$
U^+_n(x) := \sum_{k=0}^{2^n-1}e^{ikx} = \frac{e^{i2^nx}-1}{e^{ix}-1},\quad
n=0,1,2,\dots;
$$
$$
U^+_{n,k}(x) := e^{i2^nx}U^+_n(x-2\pi k2^{-n}),\quad k=0,1,\dots ,2^n-1;
$$
$$
 U^-_{n,k}(x) := e^{-i2^nx}U^+_n(-x+2\pi k2^{-n}),\quad k=0,1,\dots ,2^n-1.
$$
It will be more convenient for us to normalize in $L_2$ the system of functions
 $\{U^+_{m,k},U^-_{n,k}\}$ and  enumerate it by dyadic intervals. We write
$$
U_I(x) := 2^{-n/2}U^+_{n,k}(x)\quad \text{with}\quad
 I=[ (k+1/2)2^{-n}, (k+1)2^{-n});
$$
$$
U_I(x) := 2^{-n/2}U^-_{n,k}(x)\quad \text{with}\quad
 I=[ k2^{-n},(k+1/2)2^{-n});
$$
and
$$
U_{[0,1)}(x) :=1.
$$
Denote
$$
D^+_n:= \{I: I=[(k+1/2)2^{-n},(k+1)2^{-n}),\quad k=0,1,\dots,2^n-1\}
$$
and
$$
D^-_n:= \{ I:I=[k2^{-n},(k+1/2)2^{-n}),\quad k=0,1,\dots,2^n-1\},
$$
$$
 D_0 := [0,1),\quad
D:= \cup_{n\ge 0}(D^+_n\cup D^-_n)\cup D_0 .
$$
It is easy to check that for any $I,J \in D$, $I\neq J$ we have
$$
\<U_I,U_J\> = (2\pi)^{-1}\int_0^{2\pi}U_I(x) \bar U_J(x) dx =0,
$$
and
$$
\|U_I\|^2_2 =1.
$$

In the multivariate case of $x=(x_1,\dots,x_d)$ we define the system $\U^d$
as the tensor product of the univariate systems $\U$. Let $I=I_1\times\dots\times I_d$, $I_j \in D$, $j=1,\dots,d$, then 
$$
U_I(x) := \prod_{j=1}^d U_{I_j}(x_j) .
$$
It is known (see \cite{W}) that $\U^d$ is an unconditional basis for $L_p$, $1<p<\infty$.

We use the notations for $f\in L_1$
$$
  \hat f(k):= (2\pi)^{-d}\int_{{\mathbb T}^d}f(x)e^{-i(k,x)} dx 
$$
and for $s=(s_1,\dots,s_d) \in {\mathbb N}_0^d$
 
$$
\de_s(f) := \sum_{k\in \rho(s)}\hat f(k) e^{i(k,x)}
$$
where
$$
\rho(s) := \{k=(k_1,\dots,k_d) \in {\mathbb Z}^d: [2^{s_j-1}] \le |k_j|< 2^{s_j}, j=1,\dots,d\} .
$$

The convergence
\begin{equation}\label{3.5}
\lim_{\min_j \mu_j \to \infty}\|f-\sum_{s_j\le \mu_j,j=1,\dots,d}
\de_s(f)\|_p =0,\quad 1<p<\infty, 
\end{equation} 
 and the  Littlewood-Paley inequalities
 \begin{equation}\label{3.6}
\|f\|_p \asymp \|(\sum_s|\de_s(f)|^2)^{1/2}\|_p ,\quad 1<p<\infty, 
\end{equation}
are well-known.  

We now proceed to the key lemma of this section. 

\begin{Lemma}\label{L3.1} Let $f\in T(\bN)$. Denote $v(\bN):=\prod_{j=1}^d \Nb_j$. Then for $2\le p<\infty$ one has
$$
\Th_M(f)_p \ll v(\bN)^{1-\frac{1}{d}} (\Mb)^{-1}\|f\|_2, \quad \Mb = \max(M,1).
$$
\end{Lemma}
\begin{proof} The proof is by induction. In the case $d=2$ it follows from Lemma 2.2 of \cite{T35}. Let $d>2$. Assume $N_j=\min_i N_i.$ Represent
$$
f=\frac{1}{2N_j+1}\sum_{k=0}^{2N_j} \D_{N_j}(x_j-x_j^k)\psi_k(x^j),  
$$
where $\D_N(t)$ is the univariate Dirichlet kernel, $x_j^k=\frac{2\pi k}{2N_j+1}$, and $\psi_k(x^j) = f(x_1,\dots,x_{j-1},x^k_j,x_{j+1},\dots,x_d)$. 
Then it is well known that
$$
\|f\|_2^2 = \frac{1}{2N_j+1}\sum_{k=0}^{2N_j}  \|\psi_k(x^j)\|_2^2.
$$
By the induction assumption we obtain for $m=\sum_k m_k$
$$
\Th_m(f)_p^p \ll  \frac{1}{2N_j+1}\sum_{k=0}^{2N_j} \Th_{m_k}(\psi_k)_p^p
$$
$$
 \ll \left(\prod_{i\neq j}(2N_i+1)\right)^{(1-\frac{1}{d-1})p} (2N_j+1)^{-1}\sum_{k=0}^{2N_j} ((\mb_k)^{-1}\|\psi_k\|_2)^p.
$$
Define
$$
m_k:=\left[\frac{\|\psi_k\|_2}{\|f\|_2}\frac{M}{2N_j+1}\right].
$$
Then
$$
\sum_{k=0}^{2N_j}m_k \le \frac{M}{(2N_j+1)\|f\|_2}(2N_j+1)^{1/2}\left(\sum_{k=0}^{2N_j}\|\psi_k\|_2^2\right)^{1/2} =M.
$$
We continue
$$
\Th_M(f)_p \ll \left(\prod_{i\neq j}(2N_i+1)\right)^{(1-\frac{1}{d-1})} (2N_j+1)^{-1/p}\left(\sum_{k=0}^{2N_j} ((2N_j+1)\|f\|_2M^{-1} )^p\right)^{1/p}
$$
$$
= \left(\prod_{i\neq j}(2N_i+1)\right)^{(1-\frac{1}{d-1})} (2N_j+1)M^{-1}\|f\|_2.
$$
By our choice of $N_j$ we have
$$
\left(\prod_{i\neq j}(2N_i+1)\right)^{\frac{d-2}{d-1}} (2N_j+1) \le \left(\prod_{i=1}^d(2N_i+1)\right)^{\frac{d-1}{d}} ,
$$
which follows from
$$
(2N_j+1)^{1/d} \le \left(\prod_{i\neq j}(2N_i+1)\right)^{\frac{1}{d(d-1)}}.
$$
This completes the proof of Lemma \ref{L3.1}. 

\end{proof}
\begin{Remark}\label{R3.1} It is clear that the approximating functions $u^i_j(x_i)$ in Lemma \ref{L3.1} can be chosen from $T(N_i)$. 
\end{Remark}

{\bf Proof of Theorem \ref{T1.1}.} We consider the following class of functions which is equivalent to the class of functions with bounded mixed derivatives in $L_2$:
$$
W^r_2A:= \{f:\sum_s 2^{2r\|s\|_1}\|\de_s(f)\|_2^2 \le A^2\}.
$$
For $\|s\|_1 \le n$ set $m_s\asymp 2^{\|s\|_1(d-1)/d}$ such that for any $t\in T(\rho(s))$, $\Th_{m_s}(t)=0$. For $\|s\|_1 >n$ set 
$$
m_s:= \left[2^{(n-\kappa(\|s\|_1-n))(d-1)/d}\right]
$$
with $\kappa >0$ small enough to satisfy $r> \frac{d-1}{d} + \kappa$. Then
$$
M_1:= \sum_{\|s\|_1\le n} m_s \asymp 2^{n(d-1)/d} n^{d-1}
$$
and
$$
M_2:= \sum_{\|s\|_1>n} m_s \asymp 2^{n(d-1)/d} n^{d-1}.
$$
By Lemma \ref{L3.1} and Remark \ref{R3.1} we obtain for $M:=M_1+M_2$
$$
\Th_M(f)_p \ll \left(\sum_{\|s\|_1>n}\Th_{m_s}(\de_s(f))_p^2\right)^{1/2}
$$
$$
\ll \left(\sum_{\|s\|_1>n}(2^{-r\|s\|_1}2^{\|s\|_1(d-1)/d}(\mb_s)^{-1}\|\de_s(f)\|_2 2^{r\|s\|_1})^2\right)^{1/2}
$$
$$
\ll 2^{n(-r+(d-1)/d - (d-1)/d)}A \ll 2^{-rn} \ll \left(\frac{M}{(\log M)^{d-1}}\right)^{-\frac{rd}{d-1}}.
$$

\section{Constructive upper bounds}

In this section we discuss two algorithms for construction of good multilinear approximations. As in Section 3 we concentrate on the case $2\le p<\infty$. Our constructive upper bounds are not as good as the corresponding upper bounds for best approximations from Section 3. We begin with two main lemmas. 
\begin{Lemma}\label{L4.1} Suppose that $f\in T(\bN)$. Denote $v(\bN):=\prod_{j=1}^d \Nb_j$. Then for $1\le q\le p\le\infty$ 
\begin{equation}\label{4.1}
\Th_m(f)_p \ll v(\bN)^\bt (\mb)^{-\bt}\|f\|_q,\quad \bt:=\frac{1}{q}-\frac{1}{p},\quad \mb:=\max(1,m) .
\end{equation}
The bound (\ref{4.1}) is realized by a simple greedy-type algorithm.
\end{Lemma}
\begin{proof} In the case $1\le q\le p\le 2$, $d=2$, this lemma follows from Lemma 1.1 of \cite{T35}. That proof from \cite{T35} works in the general case  $1\le q\le p\le \infty$, $d\ge2$. We will give a sketch of this proof to illustrate the algorithm used in the construction of the approximant. 
Let $P(\bN)$ denote the set of points $z^h=(z_1^{h_1},\dots,z_d^{h_d})$, $h=(h_1,\dots,h_d)$ such that
$$
z_j^{h_j} := \frac{\pi h_j}{4{\Nb}_j},\quad h_j=0,1,\dots,8\Nb_j -1,\quad j=1,\dots,d.
$$
Denote by $\V_n(t)$ the univariate de la Vall{\' e}e Poussin kernel of order $2n-1$ for $n\ge 1$ and $\V_0(t) = 1$. Define the multivariate de la Vall{\' e}e Poussin kernel as follows
$$
\V_{\bN}(z) :=\prod_{j=1}^d \V_{N_j}(z_j),\quad \bN=(N_1,\dots,N_d).
$$
Then it is well known that any $f\in T(\bN)$ has the representation
\begin{equation}\label{4.2}
f(z) = \left(\prod_{j=1}^d (8{\Nb}_j)\right)^{-1} \sum_{z^h\in P(\bN)} f(z^h)\V(z-z^h).
\end{equation}
We have the following equivalence relation (see \cite{T35}, Theorem 1).
\begin{Theorem}\label{T4.1} For all $1\le q\le \infty$ and for $f\in T(\bN)$
$$
\|f\|_q \asymp  v(\bN)^{-1/q} \left(\sum_{z^h\in P(\bN)} |f(z^h)|^q\right)^{1/q}.
$$
\end{Theorem}

This is the Marcinkiewicz-Zygmund theorem in the case $d=1$ (see \cite{Z}, Vol II, pp. 28--33), and the general case $(d>1)$ is an immediate consequence of the one-dimensional theorem. We note that Theorem \ref{T4.1} and Lemma \ref{L4.2} (see below) hold with $P(\bN)$ replaced by a smaller net of points 
$$
P'(\bN):=\{z^h:z_j^{h_j} := \frac{\pi h_j}{2{\Nb}_j},\quad h_j=0,1,\dots,4\Nb_j -1,\quad j=1,\dots,d\}.
$$
The reader can find the corresponding results in \cite{Tb1} Chapter 2, Theorem 2.4 and Lemma 2.6. 

We also have the following inequality (see \cite{T35}, Lemma 2).
\begin{Lemma}\label{L4.2} For arbitrary numbers $A_h$
$$
\|\sum_{z^h\in P(\bN)} A_h\V(z-z^h)\|_p \ll v(\bN)^{1-1/p} \left(\sum_{h} |A_h|^p\right)^{1/p}.
$$
\end{Lemma}
  
We now complete the proof of Lemma \ref{L4.1}. Using representation (\ref{4.2}) we choose 
a set $G(m)$ of $m$ points $z^h$ with the largest $|f(z^h)|$. Then we use Theorem \ref{T4.1}, Lemma \ref{L4.2} and the following known lemma (see, for instance, \cite{T28}).
\begin{Lemma}\label{L4.3} Let $b_1\ge b_2\ge \dots b_n\ge 0$, $1\le q\le p\le \infty$ and
$$
\sum_{j=1}^n b_j^q \le A^q.
$$
Then for any $m\le n$ we have (with natural modification for $p=\infty$)
$$
\left(\sum_{j=m}^n b_j^p\right)^{1/p} \le m^{1/p-1/q}A.
$$
\end{Lemma}
It gives us
$$
\|f(z)-\sum_{z_h\in G(m)} f(z^h)\V_\bN(z-z^h)\|_p \ll v(\bN)^\bt m^{-\bt}\|f\|_q.
$$

\end{proof}

The algorithm used above in the proof of Lemma \ref{L4.1} is a simple greedy-type algorithm which uses a special dictionary $\{\V_\bN(z-z^h)\}_{z^h\in P(\bN)}$. We now proceed to a discussion of general greedy-type algorithms which will use the dictionary $\Pi^d$. We begin with a brief description of greedy approximation methods in Banach spaces. The reader can find a detailed discussion of greedy approximation in the book \cite{Tbook}. 
Let $X$ be a Banach space with norm $\|\cdot\|$. We say that a set of elements (functions) $\D$ from $X$ is a symmetric dictionary, if each $g\in \D$ has norm bounded by one ($\|g\|\le1$),
$$
g\in \D \quad \text{implies} \quad -g \in \D,
$$
and the closure of $\sp \D$ is $X$.   We denote the closure (in $X$) of the convex hull of $\D$ by $A_1(\D)$. In other words $A_1(\D)$ is the closure of conv($\D$). We use this notation because it has become a standard notation in relevant greedy approximation literature. For a nonzero element $f\in X$ we let $F_f$ denote a norming (peak) functional for $f$ that is a functional with the following properties 
$$
\|F_f\| =1,\qquad F_f(f) =\|f\|.
$$
The existence of such a functional is guaranteed by the Hahn-Banach theorem. The norming functional $F_f$ is a linear functional (in other words is an element of the dual to $X$ space $X^*$) which can be explicitly written in some cases. In a Hilbert space $F_f$ can be identified with $f\|f\|^{-1}$. In the real $L_p$, $1<p<\infty$, it can be identified with $f|f|^{p-2}\|f\|_p^{1-p}$. 
We describe a typical greedy algorithm which uses a norming functional. We call this  family of algorithms {\it dual greedy algorithms}. 
Let 
$\tau := \{t_k\}_{k=1}^\infty$ be a given weakness sequence of  nonnegative numbers $t_k \le 1$, $k=1,\dots$. We first  define the Weak Chebyshev Greedy Algorithm (WCGA) (see \cite{T15}) that is a generalization for Banach spaces of the Weak Orthogonal Greedy Algorithm.   

 {\bf Weak Chebyshev Greedy Algorithm (WCGA).} 
We define $f^c_0 := f^{c,\tau}_0 :=f$. Then for each $m\ge 1$ we have the following inductive definition.

(1) $\varphi^{c}_m :=\varphi^{c,\tau}_m \in \D$ is any element satisfying
$$
F_{f^{c}_{m-1}}(\varphi^{c}_m) \ge t_m  \sup_{g\in\D}F_{f^{c}_{m-1}}(g).
$$

(2) Define
$$
\Phi_m := \Phi^\tau_m := \sp \{\varphi^{c}_j\}_{j=1}^m,
$$
and define $G_m^c := G_m^{c,\tau}$ to be the best approximant to $f$ from $\Phi_m$.

(3) Let
$$
f^{c}_m := f^{c,\tau}_m := f-G^c_m.
$$
The index $c$ in the notation refers to Chebyshev. We use the name Chebyshev in this algorithm because at step (2) of the algorithm we use best approximation operator which bears the name of the {\it Chebyshev projection} or the {\it Chebyshev operator}. In the case of Hilbert space the Chebyshev projection is the orthogonal projection and it is reflected in the name of the algorithm. We use notation $f_m$ for the residual of the algorithm after $m$ iterations. This standard in approximation theory notation is justified by the fact that we interpret $f$ as a residual after $0$ iterations and iterate the algorithm replacing $f_0$ by $f_1$, $f_2$, and so on. In signal processing the residual after $m$ iterations is often denoted by $r_m$ or $r^m$.   

  For a Banach space $X$ we define the modulus of smoothness
$$
\rho(u) := \sup_{\|x\|=\|y\|=1}(\frac{1}{2}(\|x+uy\|+\|x-uy\|)-1).
$$
The uniformly smooth Banach space is the one with the property
$$
\lim_{u\to 0}\rho(u)/u =0.
$$
 
The following proposition is well-known (see, \cite{Tbook}, p.336).
\begin{Proposition}\label{P4.1} Let $X$ be a uniformly smooth Banach space. Then, for any $x\neq0$ and $y$ we have
\begin{equation*}
F_x(y)=\left(\frac{d}{du}\|x+uy\|\right)(0)=\lim_{u\to0}(\|x+uy\|-\|x\|)/u. 
\end{equation*}
\end{Proposition}
Proposition \ref{P4.1} shows that in the WCGA we are looking for an element $\ff_m\in\D$ that provides a big derivative of the quantity $\|f_{m-1}+u\ff_m\|$. 
Here is one more important greedy algorithm. 

  {\bf Weak Greedy Algorithm with Free Relaxation  (WGAFR).} 
Let $\tau:=\{t_m\}_{m=1}^\infty$, $t_m\in[0,1]$, be a weakness  sequence. We define $f_0   :=f$ and $G_0  := 0$. Then for each $m\ge 1$ we have the following inductive definition.

(1) $\varphi_m   \in \D$ is any element satisfying
$$
F_{f_{m-1}}(\varphi_m  ) \ge t_m   \sup_{g\in\D}F_{f_{m-1}}(g).
$$

(2) Find $w_m$ and $ \lambda_m$ such that
$$
\|f-((1-w_m)G_{m-1} + \la_m\varphi_m)\| = \inf_{ \la,w}\|f-((1-w)G_{m-1} + \la\varphi_m)\|
$$
and define
$$
G_m:=   (1-w_m)G_{m-1} + \la_m\varphi_m.
$$

(3) Let
$$
f_m   := f-G_m.
$$
It is known that both algorithms WCGA and WGAFR converge in any uniformly smooth Banach space under mild conditions on the weakness sequence $\{t_k\}$, for instance, $t_k=t$, $k=1,2,\dots$, $t>0$, guarantees such convergence. The following theorem provides rate of convergence (see \cite{Tbook}, pp. 347, 353). 
\begin{Theorem}\label{T4.2} Let $X$ be a uniformly smooth Banach space with modulus of smoothness $\rho(u)\le \gamma u^q$, $1<q\le 2$. Take a number $\ep\ge 0$ and two elements $f$, $f^\ep$ from $X$ such that
$$
\|f-f^\ep\| \le \ep,\quad
f^\ep/B \in A_1(\D),
$$
with some number $B=C(f,\ep,\D,X)>0$.
Then, for both algorithms WCGA and WGAFR  we have ($p:=q/(q-1)$)
$$
\|f_m\| \le  \max\left(2\ep, C(q,\gamma)(B+\ep)(1+\sum_{k=1}^mt_k^p)^{-1/p}\right) . 
$$
\end{Theorem}

\begin{Lemma}\label{L4.4} Let $f\in T(\bN)$. Then for $2\le p<\infty$
$$
\Th_m(f)_p \ll v(\bN)^{\frac{1}{2}-\frac{1}{pd}} (\mb)^{-1/2}\|f\|_2.
$$
The above bound is realized by the WCGA and the WGAFR with $\tau = \{t\}$. 
\end{Lemma}
\begin{proof} Assume $N_j=\max_i N_i$. Represent
$$
f(x)= \sum_{k\in Q(\bN)} {\hat f}(k) e^{i(k,x)} = \sum_{k^j\in Q(\bN^j)} u_{k^j}(x_j) e^{i(k^j,x^j)},
$$
where $k^j:=(k_1,\dots,k_{j-1},k_{j+1},\dots,k_d)$ and $x^j:=(x_1,\dots,x_{j-1},x_{j+1},\dots,x_d)$,
$$
u_{k^j}(x_j) := \sum_{|k_j|\le N_j} {\hat f}(k) e^{ik_jx_j} , 
$$
$$
Q(\bN):= \{k=(k_1,\dots,k_d): |k_i|\le N_i, i=1,\dots,d\}.
$$
Denote
$$
\psi_{k^j}(x) := u_{k^j}(x_j) e^{i(k^j,x^j)}.
$$
It is clear that $\psi_{k^j}\in \Pi^d$. We now bound
$$
\sum_{k^j\in Q(\bN^j)}\|\psi_{k^j}\|_p = \sum_{k^j\in Q(\bN^j)}\|u_{k^j}\|_p \le CN_j^{1/2-1/p}\sum_{k^j\in Q(\bN^j)}\|u_{k^j}\|_2
$$
$$
\le CN_j^{1/2-1/p}v(\bN^j)^{1/2}\left(\sum_{k^j\in Q(\bN^j)}\|u_{k^j}\|_2^2\right)^{1/2} 
$$
$$
= CN_j^{-1/p}v(\bN)^{1/2}\|f\|_2 \le Cv(\bN)^{\frac{1}{2}-\frac{1}{pd}}\|f\|_2.
$$
Therefore, 
\begin{equation}\label{4.3}
f/B\in A_1(\Pi^d_p),\quad B= Cv(\bN)^{\frac{1}{2}-\frac{1}{pd}}\|f\|_2.
\end{equation}
We proved (\ref{4.3}) for complex trigonometric polynomials. Clearly, the same proof works for real trigonometric polynomials from $T(\bN)$. We switch to real polynomials because the theory of greedy approximation, in particular the theory for the WCGA and WGAFR, is developed in real Banach spaces. We apply Theorem \ref{T4.2}. It is known that the $L_p$ space with $2\le p<\infty$ is a uniformly smooth Banach space with modulus of smoothness 
$\rho(u) \le \gamma u^2$. Applying Theorem \ref{T4.2} with $\ep=0$ and $\tau =\{t\}$ we obtain the required bound.
\end{proof}
\begin{Remark}\label{R4.1} It is clear that the approximant in Lemma \ref{L4.4} and the approximant in Lemma \ref{L4.1} in case $1<p<\infty$ can be taken from $T(\bN)$.
\end{Remark} 

\begin{Theorem}\label{T4.3} Let $2\le p <\infty$. Denote $\bt := 1/2-1/p$. Then there is a constructive way provided by Lemma \ref{L4.4} to obtain the bound
$$
\Th_M(W^r_2)_p \ll \left(\frac{M}{(\ln M)^{d-1}}\right)^{-\frac{rd}{d-1} +\frac{\bt}{d-1}},\quad r>\frac{1}{2}-\frac{1}{pd}.
$$
\end{Theorem}
\begin{proof} For $\|s\|_1 \le n$ set $m_s\asymp 2^{\|s\|_1(d-1)/d}$ such that for any $t\in T(\rho(s))$, $\Th_{m_s}(t)_p=0$. For $\|s\|_1 >n$ set 
$$
m_s:= \left[2^{(n-\kappa(\|s\|_1-n))(d-1)/d}\right]
$$
with $\kappa >0$ small enough to satisfy $r> \frac{1}{2}-\frac{1}{pd} + \kappa$. Then
$$
M_1:= \sum_{\|s\|_1\le n} m_s \asymp 2^{n(d-1)/d} n^{d-1}
$$
and
$$
M_2:= \sum_{\|s\|_1>n} m_s \asymp 2^{n(d-1)/d} n^{d-1}.
$$
By Lemma \ref{L4.4} we obtain for $M:=M_1+M_2$
$$
\Th_M(f)_p \ll \left(\sum_{\|s\|_1>n}\Th_{m_s}(\de_s(f))_p^2\right)^{1/2}
$$
$$
\ll \left(\sum_{\|s\|_1>n}(2^{-r\|s\|_1}2^{\left(\frac{1}{2}-\frac{1}{pd}\right)\|s\|_1}(\mb_s)^{-1/2}\|\de_s(f)\|_2 2^{r\|s\|_1})^2\right)^{1/2}
$$
$$
\ll 2^{n(-r+\left(\frac{1}{2}-\frac{1}{pd}\right) -\frac{1}{2}(d-1)/d)} = 2^{n(-r+\bt/d)}.
$$
\end{proof}

\begin{Theorem}\label{T4.4} Let $2\le p <\infty$. Denote $\bt := 1/2-1/p$. Then there is a constructive way provided by Lemma \ref{L4.1} to obtain the bound
$$
\Th_M(W^r_2)_p \ll \left(\frac{M}{(\ln M)^{d-1}}\right)^{-\frac{rd}{d-1} +\frac{\bt}{d-1}},\quad r>\bt.
$$
\end{Theorem}
\begin{proof} The proof of this theorem is similar to the proof of Theorem \ref{T4.3}. We use the same notations as above. Then by Lemma \ref{L4.1} we obtain for $M:=M_1+M_2$
$$
\Th_M(f)_p \ll \left(\sum_{\|s\|_1>n}\Th_{m_s}(\de_s(f))_p^2\right)^{1/2}
$$
$$
\ll \left(\sum_{\|s\|_1>n}(2^{-r\|s\|_1}2^{\bt\|s\|_1}(\mb_s)^{-\bt}\|\de_s(f)\|_2 2^{r\|s\|_1})^2\right)^{1/2}
$$
$$
\ll 2^{n(-r+\bt -\bt(d-1)/d)} = 2^{n(-r+\bt/d)}.
$$
\end{proof}

{\bf Some improvements.} In this subsection we explain how Theorems \ref{T1.1}, \ref{T4.3}, and \ref{T4.4} can be slightly improved by changing the subdivision of the set of $s$. The rest of the proofs including the use of Lemmas \ref{L3.1}, \ref{L4.1}, and \ref{L4.4} is the same. 
For a nonnegative  $(d-1)$-dimensional integer vector $w$ define
$$
S(w,j) := \{s : s_j=\max_i s_i,\, s^j=w\},\quad j=1,\dots,d.
$$
For the set $\cup_{s\in S(w,j)}\rho(s)$ we set $m_w \asymp 2^{\|w\|_1}$ in such a way that for any $t \in T\left(\cup_{s\in S(w,j)}\rho(s)\right)$ we have $\Th_{m_w}(t)_p=0$. Then
$$
M_1':= \sum_{j=1}^d \sum_{\|w\|_1\le n(d-1)/d} m_w \le 2^{n(d-1)/d}n^{d-2}.
$$
For the remaining set $S^c$ of $s$ we have
$$
S^c := \{s : s\notin \cup_{j=1}^d\cup_{\|w\|_1\le n(d-1)/d} S(w,j)\} = \{s : \min_j \|s^j\|_1 > n(d-1)/d\}.
$$
For $s\in S^c$ we define as above
$$
m_s:= \left[2^{(n-\kappa(\|s\|_1-n))(d-1)/d}\right]
$$
with $\kappa >0$ small enough. Then
$$
M_2':= \sum_{s\in S^c} m_s \le d\sum_{\|s^1\|_1\ge n(d-1)/d} \sum_{s_1\ge \|s^1\|_1/(d-1)} m_s
$$
$$
\ll \sum_{\|s^1\|_1\ge n(d-1)/d} 2^{(n-\kappa(\|s^1\|_1 + \|s^1\|_1/(d-1) -n))(d-1)/d} \ll 2^{n(d-1)/d}n^{d-2}.
$$
The rest of the proofs is the same as in Theorems \ref{T1.1}, \ref {T4.3}, and \ref{T4.4}. We only need to notice that for $s\in S^c$ we have
$$
\|s\|_1 = \frac{1}{d-1} \sum_{j=1}^d \|s^j\|_1 > n.
$$
The above argument gives us the following slightly stronger versions of Theorems \ref{T1.1}, \ref {T4.3}, and \ref{T4.4}.

\begin{Theorem}\label{T1.1'} Let $2\le p<\infty$ and $r> (d-1)/d$. Then
$$
\Th_M(W^r_2)_p \ll  \left(\frac{M}{(\log M)^{d-2}}\right)^{-\frac{rd}{d-1}}.
$$
\end{Theorem}

\begin{Theorem}\label{T4.5} Let $2\le p <\infty$. Denote $\bt := 1/2-1/p$. Then there is a constructive way provided by Lemma \ref{L4.4} to obtain the bound
$$
\Th_M(W^r_2)_p \ll \left(\frac{M}{(\ln M)^{d-2}}\right)^{-\frac{rd}{d-1} +\frac{\bt}{d-1}},\quad r>\frac{1}{2}-\frac{1}{pd}.
$$
\end{Theorem}

\begin{Theorem}\label{T4.6} Let $2\le p <\infty$. Denote $\bt := 1/2-1/p$. Then there is a constructive way provided by Lemma \ref{L4.1} to obtain the bound
$$
\Th_M(W^r_2)_p \ll \left(\frac{M}{(\ln M)^{d-2}}\right)^{-\frac{rd}{d-1} +\frac{\bt}{d-1}},\quad r>\bt.
$$
\end{Theorem}

\end{document}